\documentclass[11pt]{article}
\usepackage[a4paper,top=2.75cm,bottom=2.75cm,left=2.75cm,right=2.75cm,marginparwidth=1.75cm]{geometry}

\usepackage[utf8]{inputenc} 
\usepackage[T1]{fontenc}    
\usepackage{hyperref}       
\usepackage{url}            
\usepackage{booktabs}       
\usepackage{amsfonts}       
\usepackage{nicefrac}       
\usepackage{microtype}      

\usepackage{amsmath}
\usepackage{amsthm}
\usepackage{amssymb}
\usepackage{amsfonts}
\usepackage{bbm}
\usepackage{IEEEtrantools}

\usepackage{graphicx}
\usepackage{caption}
\usepackage{subcaption}
\usepackage{wrapfig}
\usepackage{enumitem}

\theoremstyle{definition}
\newtheorem{definition}{Definition}[section]

\newtheorem{proposition}{Proposition}[section]

\newtheorem{theorem}{Theorem}[section]
\newtheorem{lemma}{Lemma}[section]
\newtheorem{corollary}{Corollary}[section]
\newtheorem{prop}{Proposition}[section]

\newtheorem{remark}{Remark}[section]

\def\tE{{\text{E}}}

\newcommand{\one}{\mathbf{1}}

\newcommand{\h}{\mathbf{h}}

\newcommand{\mle}{\text{mle}}

\newcommand{\hR}{\widehat{R}}
\newcommand{\hr}{\hat{r}}
\newcommand{\hf}{\widehat{f}}
\newcommand{\hDelta}{\widehat{\Delta}}

\newcommand{\tT}{\tilde{T}}
\newcommand{\obs}{\text{obs}}
\newcommand{\KL}{\text{KL}}

\title{
\textbf{Optimal Estimation of Change\\ in a Population of Parameters}}
\author{Ramya Korlakai Vinayak, Weihao Kong, Sham M. Kakade\\ \\
Paul. G. Allen School of Computer Science \& Engineering, University of Washington\\
\tt{\{ramya, whkong, sham\}@cs.washington.edu}}
\begin{document}
\date{}
\maketitle

\begin{abstract}
Paired estimation of change in parameters of interest over a population plays a central role in several application domains including those in the social sciences, epidemiology, medicine and biology. In these domains, the size of the population under study is often very large, however, the number of observations available per individual in the population is very small (\emph{sparse observations}) which makes the problem challenging. Consider the setting with $N$ independent individuals, each with unknown parameters $(p_i, q_i)$ drawn from some unknown distribution on $[0, 1]^2$. We observe $X_i \sim \text{Bin}(t, p_i)$ before an event and $Y_i \sim \text{Bin}(t, q_i)$ after the event. Provided these paired observations, $\{(X_i, Y_i) \}_{i=1}^N$, our goal is to accurately estimate the \emph{distribution of the change in parameters}, $\delta_i := q_i - p_i$, over the population and properties of interest like the \emph{$\ell_1$-magnitude of the change} with sparse observations ($t\ll N$).
We provide \emph{information theoretic lower bounds} on the error in estimating the distribution of change and the $\ell_1$-magnitude of change. Furthermore, we show that the following two step procedure achieves the optimal error bounds: first, estimate the full joint distribution of the paired parameters using the maximum likelihood estimator (MLE) and then estimate the distribution of change and the $\ell_1$-magnitude of change using the joint MLE. 
Notably, and perhaps surprisingly, these error bounds are of the same order as the minimax optimal error bounds for learning the \emph{full} joint distribution itself (in Wasserstein-1 distance);
in other words, estimating the magnitude of the change of parameters over the population is, in a minimax sense, as difficult as estimating the full joint distribution itself.
\end{abstract}

\section{Introduction}\label{sec:intro}
The problem of estimating change in the parameters of interest over a population plays a key role in social sciences, epidemiology, medicine, biology, and sports analytics~\cite{wilcoxon1945individual, wacholder1982paired, moskowitz2006comparing, roozenbeek2011added, brown2008season}. Consider an example where we want to assess the impact of a policy of administering free flu vaccines on the health of a population. Suppose for a large random sample of the population we have observations of whether or not an individual caught the flu in each year for $5$ years before and $5$ years after the policy is introduced. While it is hard to accurately estimate the change in the probability of contracting flu per individual with only $5$ observations before and after, assessing the \emph{change over the population} is crucial for the public health policy makers. 
Learning the distribution and the magnitude of change is extremely useful for downstream analysis of testing and estimating properties of the distribution of change, e.g., was there a change? How much did the parameters change? Was the overall change positive?

\par One approach to estimating the change is to first learn the distributions before and after the event of interest and then estimate the change between these learned distributions. However, this approach is not sufficient in a \emph{paired observation setting}, where observations are available before and after for each individual. For example, let half the population have parameter $0.25$ and the other half $0.75$ before an event. Suppose after the event the parameter increases by $0.5$ for those with parameter $0.25$ (before) and  decreases by $0.5$ for those with parameter $0.75$ (before). Clearly, the distribution of change is supported on $\{-0.5, 0.5\}$ with half the mass on each point. However, the distribution of parameters before looks the same as that of after. In contrast, \emph{paired estimation of the distribution of change} can overcome this to recover the distribution of change.
\begin{figure*}[t!]
\centering
  \includegraphics[width=\textwidth]{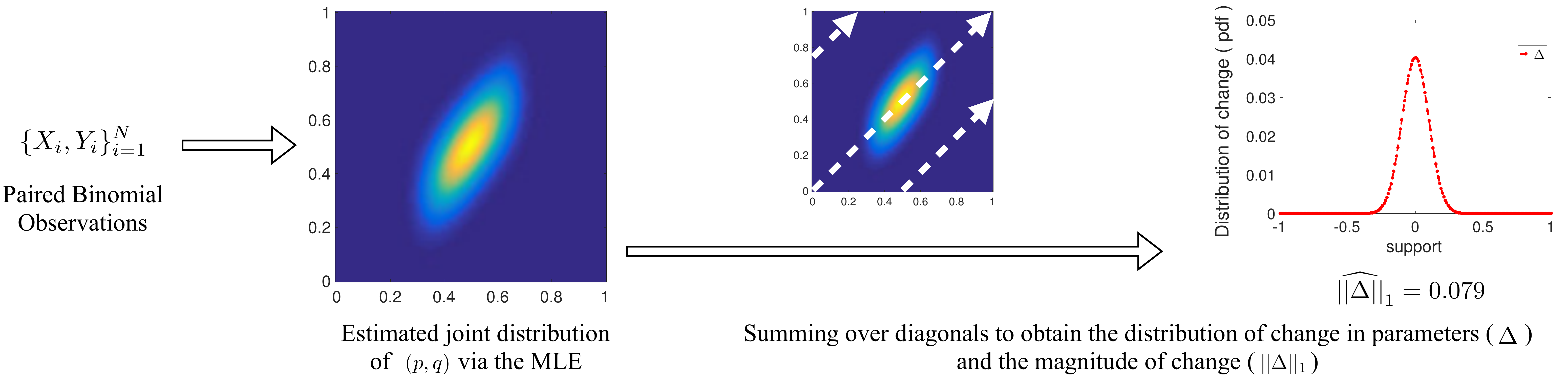}
\caption{\small{Illustration of optimal estimators for the distribution of change and the $\ell_1$-magnitude of change from paired Binomial observations obtained by using the estimated joint distribution of the parameters $(p, q)$.}}\label{fig:IlustrationdistOfChange}
\label{fig:test}
\end{figure*}
Classically, this problem is studied using \textit{paired difference test}, where the goal is to test whether the population mean or median changed after the event (e.g. introduction of a policy). \textit{Paired t-test} is typically applied when the change is normally distributed. For other cases, e.g., Binomial observations as in our setting, \emph{sign test} and \textit{Wilcoxon's signed-rank test}~\cite{wilcoxon1945individual, lehmann2006testing} are often applied. However, these tests can only answer the question of whether the population \textit{mean} or \emph{median} changed, but not how the changes distribute.
Furthermore, when a classical paired test only detects change in the mean or the median, it does not capture the \emph{effect size}, a measure of magnitude of the change~\cite{sullivan2012using}.

\paragraph{Our Contributions:} In this work, we consider the problem of estimating the distribution and the $l_1$-magnitude of change in the paired Binomial observation setting and address the following question:
\emph{What is the sample complexity for estimating the distribution and magnitude of change?}
\begin{itemize} [leftmargin=*]
    \item 
    We provide \emph{information theoretic lower bounds} on the error in estimating the distribution of change in Wasserstein-1 distance (or the earth mover's distance) and the absolute error in estimating the $\ell_1$-magnitude of the change with sparse observations. 
    \item 
    We also show that estimators derived by using the maximum likelihood estimator (MLE) of the joint distribution of the paired observations achieve the minimax optimal error bounds in the sparse observation regime (Figure~\ref{fig:IlustrationdistOfChange}). 
\end{itemize}
Perhaps surprisingly, our analysis shows that the minimax error bounds on estimating the distribution and the magnitude of the change, which is a functional of the distribution of change, is of the same order as the minimax error bound on estimating the full joint distribution of the paired observations over $[0, 1]^2$.
That is, in a minimax sense, estimating the distribution of change in the parameters over the population and it's $\ell_1$-magnitude (a functional of the distribution of change)  is as difficult as estimating the \emph{full} joint distribution of the parameters itself.
We also evaluate the empirical performance of our estimators through extensive simulations and demonstrate the use of $\ell_1$-magnitude of change for hypothesis testing.

\subsection{Problem Setup and Summary of Results}\label{sec:summaryofresults}
We can model individuals with parameters $\in [0, 1]$ in the Binomial setting as coins with biases equal to the parameters. Let $R^\star$ on $[0, 1]^2$ denote the true joint distribution of biases before and after the event. Let $N$ be the number of coins. For each coin $i$, the before and after bias tuple $(p_i, q_i)$ is drawn independently from  $R^\star$, i.e., $(p_i, q_i) \sim R^\star$. We get to observe $t$ tosses per coin before and $t$ tosses after an event. Let $\{ (X_i, Y_i) \}_{i=1}^N$ be the set of observations with $X_i \sim \text{Bin}(t, p_i)$ and $Y_i \sim \text{Bin}(t, q_i)$. $\Delta^\star$ supported on $[-1, 1]$, denotes the distribution of change in the biases after the event compared to the biases before (defined formally later, see Equation~\eqref{eqn:changedistributiontrue}). $||\Delta^\star||_1 := \int_{z=-1}^1 |z| d\Delta^\star(z)$ denotes the $l_1$-magnitude of the change. Our goal is to estimate $\Delta^\star$ and its $l_1$-magnitude $||\Delta^\star||_1$ given the paired Binomial observations $\{ (X_i, Y_i) \}_{i=1}^N$. We establish information theoretic lower bounds on how well we can estimate them over the population when $t \ll N$ (\emph{sparse observations}).
Furthermore, we construct estimators for the distribution and the magnitude of change in the biases over the population, denoted by $\hDelta$ and $\widehat{||\Delta||}_1$ respectively, which achieve the information theoretic lower bounds and therefore are minimax optimal in the sparse regime. These estimators are simple and easily to compute, derived using the MLE of the joint distribution of $(p, q)$. Figure~\ref{fig:IlustrationdistOfChange} shows an illustration of integrating over the diagonal and sub-diagonals of the joint distribution over $[0, 1]^2$ to obtain the distribution of change supported on $[-1, 1]$. These estimators are defined formally later in Definitions~\eqref{def:distofChange} and~\eqref{def:MagofChange}. 
We bound the error in the estimation of the distribution of change (under Wasserstein-1 distance, Definition~\ref{def:W1def}) and the absolute error in estimating the $l_1$-magnitude of change and in conjunction with the lower bounds show that,
\begin{theorem}(Informal)
\begin{itemize}[leftmargin=*, topsep=-2pt, itemsep=-2pt]
    \item 
    When $t < c\ \log{N}:\ W_1( \Delta^\star , \widehat{\Delta} ) = \Theta_{\alpha}\left( \frac{1}{t} \right);\ \left| ||\Delta^\star||_1 - \widehat{||\Delta||_1} \right|= \Theta_{\alpha}\left( \frac{1}{t} \right).\footnote{$\Theta_\alpha(.)$ hides $\log{(1/\alpha)}$ in the bound for it to hold with probability at least $1-2\alpha$.}
$
\item When $ t \in \left[ \Omega(\log{N}), \mathcal{O}\left({N^{1/8 - \epsilon}} \right) \right]:\ 
     W_1( \Delta^\star , \widehat{\Delta} )= \Theta_\alpha\left(\frac{1}{\sqrt{t \log{N}}} \right);\ \left| ||\Delta^\star||_1 - \widehat{||\Delta||_1} \right|= \Theta_\alpha\left(\frac{1}{\sqrt{t \log{N}}} \right).
$
\end{itemize}
\end{theorem}
\vspace{-5pt}
Note that $t < c\ \log{N}$ implies that the size of the population, $N \gtrsim \exp{(t)}$. When $ t \in \left[ \Omega(\log{N}), \mathcal{O}\left({N^{1/8 - \epsilon}} \right) \right]$, then $ \text{poly}(t) \lesssim N \lesssim \exp{(t)}$.

\subsection{Related Works}\label{sec:related works}
\paragraph{Estimating the distribution of population of parameters:} The MLE for estimating the unknown distribution of parameters in general nonparametric mixture model setting has been studied extensively in the literature~\cite{lord1975empirical,turnbull1976empirical, simar1976maximum, laird1978nonparametric, cressie1979quick, lindsay1983geometry, lindsay1983geometry2, bohning1989likelihood, lesperance1992algorithm}. There have been several works that focus on the Binomial model~\cite{lindsay1983geometry, lindsay1983geometry2, wood1999binomial}. Most of the prior literature focused on understanding the geometry of the MLE, identifiability and uniqueness of the solutions, and optimality criteria of the MLE. Recently, Vinayak et. al~\cite{vinayak2019maximum} studied the Wasserstein error rate of the MLE in the one dimensional setting, and showed that the MLE achieves error $\mathcal{O}(\frac{1}{t})$ in with $t=\mathcal{O}(\log{N})$ coin tosses and $\mathcal{O}(\frac{1}{\sqrt{t\log{N}}})$ in $t=\Omega(\log{N})$ regime, which are minimax optimal in both cases. Li et. al~\cite{li2015learning} proposed a linear programming based estimator that achieves an error bound of $\mathcal{O}(\frac{1}{t})$ when $t = \mathcal{O}(\log{N})$ and a suboptimal error bound of $\mathcal{O}(\frac{1}{\sqrt{t}})$ otherwise. Tian et. al~\cite{tian2017learning} proposed a method of moments estimator to estimate the unknown distribution of biases in the regime where the number of coin tosses is $t = O(\log N)$. In the multivariate setting where the biases are drawn from a distribution supported on $[0,1]^d$ with constant $d$, their estimator achieves Wasserstein error $O(\frac{1}{t})$ which is minimax optimal. The main weakness of this method of moments approach is that it fails to obtain the optimal rate when $t > c \log{N}$. We will comment on using the joint distribution learned using moment matching estimator in Remark~\ref{remark:moment}.

\paragraph{Symmetric properties of discrete distributions:} Estimating the $\ell_1$-magnitude of change is a special case of functional estimation problems where the goal is to estimate a certain functional of the parameters. Several recent works~\cite{orlitsky2004modeling, acharya2009recent, acharya2010exact, valiant2011estimating, valiant2013estimating,jiao2015minimax, wu2015chebyshev, valiant2016instance, wu2016minimax, orlitsky2016optimal, acharya2017unified, jiao2018minimax, Han2018} have focused on estimating symmetric functionals of discrete distributions, such as entropy and support size. Many of these works have observed that some functionals can be accurately estimated even in the setting where the number of samples is not sufficient to reliably learn the distribution. For example, Valiant and Valiant~\cite{valiant2011estimating} showed that accurate estimation of the entropy of a discrete distribution with support size $n$ only requires $O(\frac{n}{\log n})$ samples while learning the distribution itself requires at least $\Omega(n)$ samples. 

\paragraph{PML framework:} A series of works~\cite{acharya2009recent,acharya2010exact,acharya2017unified} examined the \textit{profile} or \textit{pattern} maximum likelihood (PML) as a unifying framework for estimating symmetric properties of a discrete distribution. A limitation of this framework is that it only applies to scalar symmetric properties of discrete distributions. To the best of our knowledge, it is not known how to apply these approaches to estimating asymmetric properties, for examples $\ell_1$ distance between two discrete distributions which is related to the setting we considered. Another limitation is that the near-optimal sample complexity of using the PML for estimating a function estimation relies on the existence of another stable estimator of the function, and the PML itself does not give the minimax rate. In contrast, we establish the optimality of using derived estimators for the distribution as well as $\ell_1$ magnitude of change by providing matching information theoretic lower bounds for these estimation problems. Furthermore, the setting in our work is different from that of discrete distribution setting (as emphasized in the following paragraph).

\paragraph{Estimating $\ell_1$-distance:} Jiao et. al~\cite{jiao2018minimax} considers the problem of estimating the $\ell_1$-distance between two discrete distributions. There are two major differences between this work and the results in~\cite{jiao2018minimax}. First, Jiao et. al~\cite{jiao2018minimax} considered the discrete distribution setting where the parameters sum to $1$ and the observations follow Poisson distribution. In contrast, in the setting considered in this work, the observations are Binomial and the parameters need not sum to $1$. Second, on the algorithmic side, the estimator in Jiao et. al~\cite{jiao2018minimax} is based on a multivariate piece-wise polynomial function that is specifically designed to approximate the function $|y-x|$. In contrast, our estimator for the $\ell_1$-magnitude of change is constructed by using the MLE solution of the joint distribution, which can be generalized to estimate other functionals.
\section{Main Results}\label{sec:mainresults}
In this section we present our main results. We start with some notations and definitions that help us in presenting the formal results.
Recall that we get $N$ paired binomial observations, $\left\{(X_i, Y_i) \right\}_{i=1}^N$, where $X_i \sim \text{Bin}(t, p_i)$ and $Y_i \sim \text{Bin}(t, q_i)$ are the number of heads seen before and after by coin $i$. The binomial parameters $(p_i, q_i)$ are drawn i.i.d. from the joint distribution $R^\star$. Let $n_{u, v}:= \sum_{i=1}^N\one{\{(X_i, Y_i) = (u, v)\}}$, denote the number of coins that sees tuple $(u, v)$, that is, turned up $u$ heads before and $v$ heads after out of $t$ tosses. Let $h_{u, v} = n_{u, v}/N$ denote the fraction of coins that show $(u, v)$ observation tuple. The matrix $H \in \mathbb{R}^{(t+1)\times (t+1)}_{\geq 0}$, with $H_{u, v} = h_{u, v}$ is the \emph{fingerprint} matrix for the joint observation.
\begin{definition}(\textbf{Distribution of Change})\label{def:distofChange}
Define random variable $\delta = q-p$ where $(p,q) \sim R^\star$. Let $\Delta^\star$ be the distribution of $\delta$, i.e., the true distribution of change in the biases over $[-1, 1]$. Then $\Delta^\star$ can be defined using $R^\star$ as follows, 
for $w \in [-1, 1], $
\begin{align}
   \Delta^\star(w) &:= \int_{z=-1}^w \int_{x=\max{\{0, -z\}}}^{\min{\{1, 1-z\}}} r^\star(x, x+z)\ dx\ dz =: \int_{z=-1}^w \delta^\star(z)\ dz,\label{eqn:changedistributiontrue}
\vspace{-10pt}
\end{align}
where $r^\star$ and $\delta^\star$ denotes the probability density functions (pdf) of the distributions $R^\star$ and $\Delta^\star$ respectively.
\end{definition}
\begin{definition}(\textbf{Magnitude of Change})\label{def:MagofChange}
The $l_1-$magnitude of the change in the biases, denoted by $||\Delta^\star||_1$, is given by the following equation,
\begin{align}
||\Delta^\star||_1 &= \int_{x=0}^1 \int_{y=0}^1 |y-x|\ r^\star(x, y)\ dx\ dy = \int_{z=-1}^1 |z|\ \delta^\star(z)\ dz. \label{eqn:l1ofchange}
\end{align}
\end{definition}
\paragraph{Goal:} Given the paired observations $\{(X_i, Y_i)\}_{i=1}^N$, our goal is to estimate the distribution $\Delta^\star$ and its $l_1-$ magnitude $||\Delta^\star||_1$. 

The accuracy of the estimator of the distribution of change is measured in \emph{Wasserstein-1 distance} or the \emph{Earth Mover's Distance} (EMD) defined below:
\begin{definition}(\textbf{Wasserstein-1 distance}, dual definition~\cite{Kantorovich1958})\label{def:W1def}
The Wasserstein-1 distance or the Earth Mover's Distance (EMD) between two distributions $P$ and $Q$ supported on $\mathcal{M}$ is defined as,
\begin{equation}
W_1(P, Q) = \underset{f \in \text{Lip}(1)}{\text{sup}} \int_{\mathcal{M}} f(x) (dP(x) - dQ(x)),\label{eqn:W1definition}
\end{equation}
where $\text{Lip}(1)$ denotes the set of all Lipschitz-1 functions over $\mathcal{M}$.
\end{definition}

\subsection{Information Theoretic Lower Bounds}\label{sec:minimaxlowerbounds}
The following theorem provides information theoretic lower bounds on the error in estimating the the $l_1-$magnitude of change and the distribution of change under the Wasserstein-1 distance.
\begin{theorem}\label{prop:lowerboundEstDistChange} \label{prop:MinimaxLowerBounds}
Let $R$ be a distribution over $[0, 1]^2$. Let $\mathbf{S} := \{(X_i,Y_i)\}_{i=1}^N$ be random variables with $X_i \sim \text{Binomial}(t, p_i)$, $Y_i \sim \text{Binomial}(t, q_i)$ where $(p_i,q_i)$ is drawn independently from $R$. Define random variable $\delta = q-p$ where $(p,q) \sim R$, and let $\Delta$ be the distribution of $\delta$.
\begin{enumerate}[leftmargin=*, itemsep=0pt]
\item Define $||\Delta_R||_1 = \tE_{(p,q)\sim R}[|q-p|]$. Let $\zeta$ be an estimator that maps $\mathbf{S}$ to a real value $\zeta(\mathbf{S})$. Then for every $t, N$ s.t. $t\le \frac{N^{2(e^4-1)}}{36}$,
\begin{align*}
    \small{\underset{\zeta}{\inf}\ \underset{R}{\sup}\ \text{E}\left[ |  ||\Delta_R||_1-\zeta(\mathbf{S}))| \right]
    =\Omega\left( \frac{1}{t}\right) \vee \Omega\left(\frac{1}{\sqrt{t\log N}} \right)}.
\end{align*}
\item  Let $g$ be an estimator that maps $\mathbf{S}$ to a probability distribution $g(\mathbf{S})$ supported on $[-1,1]$. Then for every $t, N$ s.t. $t\le \frac{N^{2(e^4-1)}}{36}$,
\begin{equation*}
    \underset{g}{\inf}\ \underset{R}{\sup}\ \text{E}\left[W_1(\Delta,g(\mathbf{S}))\right] = \Omega\left( \frac{1}{t}\right) \vee \Omega\left(\frac{1}{\sqrt{t\log N}} \right),\footnote{$a \vee b := \max{\{ a, b\}}$.}
\end{equation*}
\end{enumerate}
\end{theorem}
In the setting in Theorem~\ref{prop:lowerboundEstDistChange}, the following proposition provides the information theoretic lower bound on the Wasserstein-1 error in estimating the joint distribution of $(p, q)$.
\begin{corollary}(Lower bound on estimating the joint distribution.)\label{prop:lowerboundjointMLE} 
Let $\rho$ be an estimator that maps $\mathbf{S}$ to a probability distribution $\rho(\mathbf{S})$ supported on $[0,1]^2$. Then for every $t, N$ s.t. $t\le \frac{N^{2(e^4-1)}}{36}$,
\begin{equation*}
    \underset{\rho}{\inf}\ \underset{R}{\sup}\ \text{E}\left[W_1(R, \rho(\mathbf{S}))\right] = \Omega\left( \frac{1}{t}\right) \vee \Omega\left(\frac{1}{\sqrt{t\log N}} \right).
\end{equation*}
\end{corollary}

\subsection{Estimators for the Distribution and the Magnitude of Change}\label{sec:estimators}
\vspace{-5pt}
Our estimators for the distribution and the magnitude of change are derived from the MLE of the joint distribution of the parameters $(p, q)$ given the observations. Given the paired Binomial observations $\{(X_i, Y_i)\}_{i=1}^N$, the MLE of the joint distribution of $(p, q)$ denoted by $\hat{R}_\mle$ is as follows,
\begin{align}
\hat{R}_\mle &\in \underset{S \in \mathcal{S}}{\text{arg max}}\ \frac{1}{N} \log{ \Pr\left( \{ (X_i, Y_i) \}_{i=1}^N\ |\ S \right)}, \label{eqn:jointMLEobj0}\\
&= \underset{S \in \mathcal{S}}{\text{arg max}}\ \frac{1}{N} \sum_{i=1}^N \log{ \Pr\left(  (X_i, Y_i)\ |\ S \right)},\nonumber\\
	& =  \underset{S \in \mathcal{S}}{\text{arg max}}\ \frac{1}{N} \sum_{u=0}^t \sum_{v=0}^t n_{u,v} \log{  \Pr\left(  (u, v)\ |\ S \right)}, \nonumber
\end{align}
where $\mathcal{S}$ is the set of all distributions on $[0, 1]^2$. Note that $\Pr\left(  (u, v)\ |\ S \right)$ can be written as,
\begin{align}
    &\int_0^1 \int_0^1 \binom{t}{u} x^{u} (1 - x)^{t-u} \binom{t}{v} y^{v} (1 - y)^{t-v} dS(x, y)\nonumber =: \tE_S\left[ h_{u, v} \right],
\end{align}
where $E_S[h_{u,v}]$ is the expected fingerprint, i.e., the fraction of the population that sees the tuple $(u, v)$ under the distribution $S$. Therefore, the MLE of the joint distribution of $(p, q)$ is written in terms of the expected fingerprints as follows,
\begin{align}
\hat{R}_\mle &\in \underset{S \in \mathcal{S}}{\text{arg max}}\ \sum_{u=0}^t \sum_{v=0}^t h_{u,v} \log{ \tE_S\left[ h_{u, v} \right]}.\label{eqn:MLEObjFingerprint}
\end{align}
We note that the objective function is concave in the distribution $S$ and it is strictly concave in the expected fingerprints. Therefore, while there can be multiple distributions maximizing the objective function~\eqref{eqn:MLEObjFingerprint}, the expected fingerprints under them will be the same.

\paragraph{Estimator for the distribution of change $\Delta^\star$} over the population is denoted by $\hDelta$ and is derived using the joint MLE is as follows, for $w \in [-1, 1]$,
\begin{equation}
 \widehat{\Delta}(w) := \int_{z=-1}^w \int_{x=\max{\{0, -z\}}}^{\min{\{1, 1-z\}}} \widehat{r}_\mle (x, x+z)\ dx\ dz,\label{eqn:estchangedist1}
\end{equation}
where $\hat{r}_\mle$ is the probability density function of $\hR_\mle$. 

\paragraph{Estimator for the magnitude of change $||\Delta^\star||_1$} over the population, denoted by $\widehat{||\Delta||}_1$, is derived from the joint MLE as follows:
\begin{align}
    \widehat{||\Delta||}_1 &:= \int_{x=0}^1 \int_{y=0}^1 |y-x|\ \hr_\mle(x, y)\ dx\ dy = \int_{z=-1}^1 |z|\ \hat{\delta}(z)\ dz, \label{eqn:estl1ofchange}
\end{align}
where $\hat{\delta}$ is the probability density function of $\widehat{\Delta}$. 

With the setting in Section~\ref{sec:summaryofresults},  the following theorem provides upper bounds on the accuracy of the estimators above:
\begin{theorem}(Estimation of the Distribution of Change)\label{thm:upperbounds}
There exists $\epsilon>0$ such that with probability at least $1 - 2\alpha$,
\begin{enumerate}[leftmargin=*, itemsep=0pt]
\item The Wasserstein-1 error incurred by the estimator in Equation~\eqref{eqn:estchangedist1} can be upper bounded as follows,
\begin{equation*}
W_1( \Delta^\star,  \widehat{\Delta}) = 
\begin{cases}
 \mathcal{O}_{\alpha}\left( \frac{1}{t} \right), \text{ when } t = \mathcal{O}(\log{N}),\\
\mathcal{O}_\alpha\left(\frac{1}{\sqrt{t \log{N}}} \right), \text{ when } t \in \left[ \Omega(\log{N}), \mathcal{O}\left({N^{1/8 - \epsilon}} \right) \right].
\end{cases}
\end{equation*}
\item The error incurred by the estimator in Equation~\eqref{eqn:estl1ofchange} can be upper bounded as,
\begin{align*}
 \left| ||\Delta^\star||_1 - \widehat{||\Delta||}_1 \right| =\begin{cases}
 \mathcal{O}_{\alpha}\left( \frac{1}{t} \right), \text{ when } t = \mathcal{O}(\log{N}),\\
\mathcal{O}_\alpha\left(\frac{1}{\sqrt{t \log{N}}} \right), \text{ when } t \in \left[ \Omega(\log{N}), \mathcal{O}\left({N^{1/8 - \epsilon}} \right) \right].
\end{cases}
\end{align*}
\end{enumerate}
\end{theorem}
Along with the lower bound in Theorem~\ref{prop:MinimaxLowerBounds}, the above upper bounds show that our estimators for the distribution of change and the $l_1$-magnitude of change given in Equation~\eqref{eqn:estchangedist1} and Equation~\eqref{eqn:estl1ofchange} respectively are optimal in the sparse observation regime.
 
With the setting in Section~\ref{sec:summaryofresults}, the following proposition bounds the estimation error for the joint MLE.
\begin{prop}(MLE for the Joint Distribution)\label{prop:MLEJointW1}
There exists $\epsilon>0$ such that with probability at least $1 - 2\alpha$,\vspace{-5pt}
\begin{equation*}
W_1( R^\star,  \widehat{R}_\mle) = 
\begin{cases}
 \mathcal{O}_{\alpha}\left( \frac{1}{t} \right), \text{ when } t = \mathcal{O}(\log{N}),\\
\mathcal{O}_\alpha\left(\frac{1}{\sqrt{t \log{N}}} \right), \text{ when } t \in \left[ \Omega(\log{N}), \mathcal{O}\left({N^{1/8 - \epsilon}} \right) \right].
\end{cases}
\end{equation*}
\end{prop}
Along with the lower bound in Proposition~\ref{prop:lowerboundjointMLE}, the above upper bound shows that the MLE is minimax optimal for estimating the joint distribution in the sparse observation regime.
Furthermore, we note that the upper bound on the error in estimating the joint distribution of the parameters $(p, q)$ is of the same order as the lower bounds on error in estimating the distribution and magnitude of change. Even though the $l_1$-magnitude of change is a scalar value, estimating it is no easier than estimating the whole joint distribution itself.

\begin{remark}\normalfont{The standard {\textbf{naive plug-in empirical estimator}} which uses the empirical histogram of the observations, $\{ \left(\frac{X_i}{t},\frac{Y_i}{t} \right) \}_{i=1}^N$, would be sub-optimal in estimating the quantities of interest in the sparse observation regime since it incurs an error of $\Theta\left(\frac{1}{\sqrt{t}}\right)$.} 
\end{remark}

\begin{remark}\label{remark:moment}
\normalfont{\textbf{Moment based estimation:} 
Our estimators for the distribution and $l_1$-magnitude of change in Equations~\eqref{eqn:estchangedist1} and~\eqref{eqn:estl1ofchange} can be easily modified by replacing $r_\mle$ with the estimated joint distribution of $(p, q)$ using moment matching~\cite{tian2017learning} to obtain corresponding estimators derived by moment matching. These derived estimators would achieve the minimax optimal error rate of $\Theta\left(\frac{1}{t}\right)$ when $t < c\ \log{N}$ (i.e., when $N \gtrsim \exp{(t)}$). However, they fail to achieve optimal error rates when $t > c\ \log{N}$ due to higher variance in larger moments.}
\end{remark}
\begin{remark}\normalfont{
\textbf{Estimating change with only observation of change:} 
Binomial distribution is not transnational invariant, therefore, observing only the change, $Z = Y-X$, is not enough to estimate the distribution of change in parameters, $\delta = q - p$. Unbiased estimate of the average change in the parameters, $\tE[\delta]$, can be obtained using only the difference, $\frac{1}{N}\sum_{i=1}^N \frac{Z_i}{t}$. However, unbiased estimates of the higher moments would need the joint moments. For example, an unbiased estimate of the second moment $\tE[\delta^2]$ would be 
\[ \widehat{\tE[\delta^2]} := \frac{1}{N} \sum_{i=1}^N  \left( \frac{\binom{X_i}{2}}{\binom{t}{2}}+  \frac{\binom{Y_i}{2}}{\binom{t}{2}} -  2\frac{X_i Y_i}{t^2}\right), \]
which involves all the estimated joint second moments of $(p, q)$. In contrast, a naive empirical estimate using just the change $Z$ gives rise to a biased estimate, $ \frac{1}{N}\sum_{i=1}^N \left(\frac{Y_i-X_i}{t}\right)^2 = \frac{1}{N}\sum_{i=1}^N \delta_i^2 + \frac{1}{N}\sum_{i=1}^N \frac{q_i(1-q_i) + p_i(1-p_i)}{t}$. For small $t$, the second term is not zero unless all $(p_i, q_i)$ are 0 or 1.
}
\end{remark}
\section{Proof Sketches for Main Results}\label{sec:proofsketches}
In this section we provide a brief sketch of the proofs the main results.
\paragraph{Proof sketches for lower bounds: }
To show the information theoretic lower bound on the error in estimating the $\ell_1$-magnitude of change, we use the following proposition.
\begin{proposition}\label{prop:prior-exist-magnitude}
For any integer $s > 0$, $\exists$ a pair of distributions $P$ and $Q$ supported on $[a, b]$ where $0 < a < b$, such that their first $s$ moments match, and \[ \left|\int_a^b \left|x-\frac{a+b}{2}\right|P(x)dx -\int_a^b \left|x-\frac{a+b}{2}\right|Q(x)dx \right|= \Omega\left(\frac{b-a}{s}\right). \]
\end{proposition}
Using Proposition~\ref{prop:prior-exist-magnitude} we construct two distributions $P^*$ and $Q^*$ supported on $[0, 1]$ such that their first $t$ moments match and, 
\[\int_0^1 \left|x-\frac{1}{2}\right|P(x)dx - \int_0^1 \left|x-\frac{1}{2}\right|Q(x)dx \ge \Omega\left(\frac{1}{t}\right).\]
Let $U=\delta_{1/2}$ be a distribution supported only at $1/2$.
Define $2$-dimensional distribution $P = U\times P^*, Q = U\times Q^*$. We first show that the difference in the $\ell_1$-magnitude of change of these two joint distributions is lower bounded by $1/t$. We then argue that if the tuple $(p, q)$ are drawn from $P$ or $Q$ is information theoretically not distinguishable from the binomial observations $(X, Y)$. To show the lower bound for the case when $t > c\ \log{N}$, we construct distributions similar to above supported on a small interval around $1/2$ whose first $\log{N}$ moments match while the $\ell_1$-magnitudes of change differ by $\frac{1}{\sqrt{t\ \log{N}}}$. The lower bounds on the error in estimating the $\ell_1$-magnitude of change imply the lower bounds for estimating the distribution of change in the parameters as well.

\paragraph{Proof Sketches for the Upper Bounds:}
Using the definition of $\Delta^\star$ (Equation~\eqref{eqn:changedistributiontrue}) and changing variables from $(x, z:= y - x)$ to $(x, y)$,  
the Wasserstein-1 distance between the true distribution of change $\Delta^\star$ and the estimated change distribution $\hDelta$ can be written as follows,
\begin{align}
W_1( \Delta^\star,  \hDelta)=\underset{f \in \text{Lip}(1)}{\text{sup}} \int_{0}^1 \int_{0}^1 f(y-x) \left(dR^\star(x,y) - d\hR_\mle(x,y)\right).\label{eqn:W1ofDistofChange}
\end{align}
The estimation error for the magnitude of change is given by:
\begin{align}
 \left|||\Delta^\star||_1 - \widehat{||\Delta||}_1\right| =\left|\int_0^1 \int_0^1 |y-x| (dR^\star(x,y) - d\hR_\mle(x,y)) \right|.\label{eqn:ErrorOfMagofChange}
\end{align}
By the definition the Wasserstein-1 distance between the true joint distribution $R^\star$ and $\hR_\mle$ is as follows,
\begin{align}
    W_1( R^\star,  \widehat{R}_\mle) =\underset{f \in \text{Lip}(1)}{\text{sup}} \int_{0}^1 \int_{0}^1 f(x, y) (dR^\star(x, y) - d\hR_\mle(x, y)).\label{eqn:W1ofJointDist}
\end{align}
Noting that $|y-x|$ is a Lipschitz-1 function, all the estimator error expressions in Equations~\eqref{eqn:W1ofDistofChange},~\eqref{eqn:ErrorOfMagofChange} and~\eqref{eqn:W1ofJointDist} involve Lipschitz-1 functions. It follows from the definition that,
\begin{align*}
\left|||\Delta^\star||_1 - \widehat{||\Delta||}_1\right| \le W_1( \Delta^\star,  \hDelta) \le W_1( R^\star,  \widehat{R}_\mle).
\end{align*}

The following bound holds for any $f \in \text{Lip}(1)$ supported on $[0, 1]^2$ and any polynomial approximation $\hf$ of $f$ on $[0, 1]^2$. 
\begin{align}
&\left| \int_{x=0}^1 \int_{y=0}^1 f(x, y) (dR^\star(x, y) - d\hR_\mle(x, y)) \right|\nonumber\\
&= \left| \int_{0}^1 \int_{0}^1 \left(f(x, y) - \hf(x,y) \right) (dR^\star - d\hR_\mle) +  \int_{0}^1 \int_{0}^1 \hf(x,y) (dR^\star - d\hR_\mle) \right| ,\nonumber\\
&\leq 2||f - \hf||_\infty + \left|\tE_{ R^\star}[\hf] - \tE_{ \widehat{R}_\mle}[\hf]  \right|,\label{eqn:bound2dintegral}
\end{align}
where $||f - \hat{f}||_\infty := \underset{(x,y) \in [0, 1]^2}{\max}|f(x, y) - \hf(x, y)|$ is the uniform polynomial approximation error. Let $\hf$ be the 2-d Bernstein polynomial of degree $2t$ defined by,
\begin{equation}
     \hf_t(x, y) = \sum_{u=0}^t \sum_{v=0}^t \kappa_{uv} B_{u}^t(x)  B_{v}^t(y), \label{eqn:BernsteinPolynomialApprox2d}
\end{equation}
where $B_{u}^t(x) := \binom{t}{u}x^u (1-x)^{t-u}$ is the $u-$th Bernstein polynomial of degree $t$. 
Using this in Equation~\eqref{eqn:bound2dintegral}, the EMD between the true joint distribution of $(p, q)$, $R^\star$ and the MLE estimator $\hR_\mle$ can be bounded as,
\begin{align}
W_1(R^\star, \hR_{\mle}) &\leq \underset{f \in \text{Lip}(1)}{\text{sup}} \left\{ 2 ||f - \hf||_\infty + \left| \sum_{u=0}^t \sum_{v=0}^t \kappa_{uv} \left(\tE_{ R^\star}[h_{uv}] - \tE_{ \widehat{R}_\mle}[h_{uv}]\right)  \right|  \right\},\label{eqn:W1JointBound}
\end{align}
This upper bound also holds for Equations~\eqref{eqn:W1ofDistofChange} and~\eqref{eqn:ErrorOfMagofChange}.
The second term in the bound above Equation~\eqref{eqn:W1JointBound}
can be bounded by the following expression,
\begin{align*}
&\left| \sum_{u=0}^t \sum_{v=0}^t \kappa_{uv} \left(\tE_{ R^\star}[h_{uv}] - \tE_{ \widehat{R}_\mle}[h_{uv}]\right)  \right| 
\overset{(a)}{\leq} \mathcal{O}\left( \max_{u,v} |\kappa_{uv}| \sqrt{\frac{(t+1)^2}{2N} \log{\frac{4N}{(t+1)^2}} + \frac{1}{N} \log{\frac{3e}{\alpha}} }\right),
\end{align*}
where the inequality $(a)$ uses triangle inequality followed by concentration of fingerprints and bound on the error in 2-d fingerprint due to MLE solution. Note that the term $ \max_{(u,v)} |\kappa_{uv}|$ appears in the bound above. Therefore, it is crucial to bound the uniform approximation error $||f - \hf||_\infty$ while keeping the coefficients of the approximation~\eqref{eqn:BernsteinPolynomialApprox2d} from being too large. Polynomial approximation of Liptschitz-1 functions on $[0, 1]^2$ using Bernstein polynomials play a crucial role in obtaining a tight bound that match the information theoretic lower bounds.
We extend the analysis of~\cite{vinayak2019maximum} to obtain these bounds. The details are available in the appendix.
\section{Empirical Performance}\label{sec:simulations}
\paragraph{Solving the MLE:} Recall the objective function for the joint MLE in Equation~\ref{eqn:MLEObjFingerprint}. The objective function is concave in the distribution $S$ and strictly concave in the expected fingerprints. Furthermore, the set $\mathcal{S}$ of all distributions over $[0, 1]^2$ is convex. Thus the MLE of the joint distribution in this setting is an optimization problem of maximizing a concave function over a convex set. We solve this optimization problem on a $100 \times 100$ uniform grid on $[0, 1]^2$. The objective function of the joint MLE on a $m \times m$ grid on $[0, 1]^2$ is as follows,
\[F(S) :=  \sum_{u = 0}^t \sum_{v=0}^t h_{u, v}^\text{obs} \log{ \left(\sum_{i = 0}^m \sum_{j=0}^m \binom{t}{u} \left(\frac{i}{m}\right)^{u} \left(1 - \frac{i}{m}\right)^{t-u} \binom{t}{v} \left(\frac{j}{m}\right)^{v} \left(1 - \frac{j}{m}\right)^{t-v} S_{ij}\right)}. \]
So, the MLE objective is to maximize $F(S)$ for $S$ on the simplex, that is, $S_{ij} \geq 0 $ and $\sum_i \sum_j S_{ij} = 1$. Let $\text{vec}(S)$ denote vectorized version of the matrix $S$ by stacking the columns and $\triangledown_{S}F$ denote the gradient of $F$ at $S$. We use exponential gradient descent algorithm (or equivalently, mirror descent with KL divergence) to optimize this objective as described below:
\begin{enumerate}
    \item Input: Observed fingerprints $\{h_{u,v}^\text{obs}\}_{u, v = 0}^t$, learning rate $\eta$, grid size $m$, tolerance $\varepsilon$, maximum iterations: $\textit{MaxIters}$.
    \item Start with $S(0)$ being Uniform distribution on $[0,1]^2$, i.e., $[S(0)]_{ij} = \frac{1}{(m+1)^2}$.
    \item For $\text{iter} = 0, 1, 2, ...., \textit{MaxIters}$
    \begin{itemize}
        \item Update: $[S(\text{iter}+1)]_{ij}$ = $[S(\text{iter})]_{ij}\ \exp{\{\eta\ \left[\triangledown_{S(\text{iter})}F\right]_{ij} \}}$, for $i, j = 0,...., m$, where 
        \item Terminate if $||\text{vec}(S(\text{iter}+1)) - \text{vec}(S(\text{iter}))||_2 < \varepsilon$.
    \end{itemize}
\end{enumerate}
We run the above algorithm with a learning rate $0.5$, stopping criteria of $||\text{vec}(s_{iter}) - \text{vec}(s_{{iter}-1})||_2 < 5\times10^{-7}$ and maximum iterations set to $2\times10^4$. Estimates $\hDelta$ and $\widehat{||\Delta||}_1$ are obtained using the discrete joint MLE solution. All the experiments were run on Matlab\textsuperscript{\textregistered}R2017b.
\begin{figure*}[t!]
 \centering
    \includegraphics[width=0.9\textwidth]{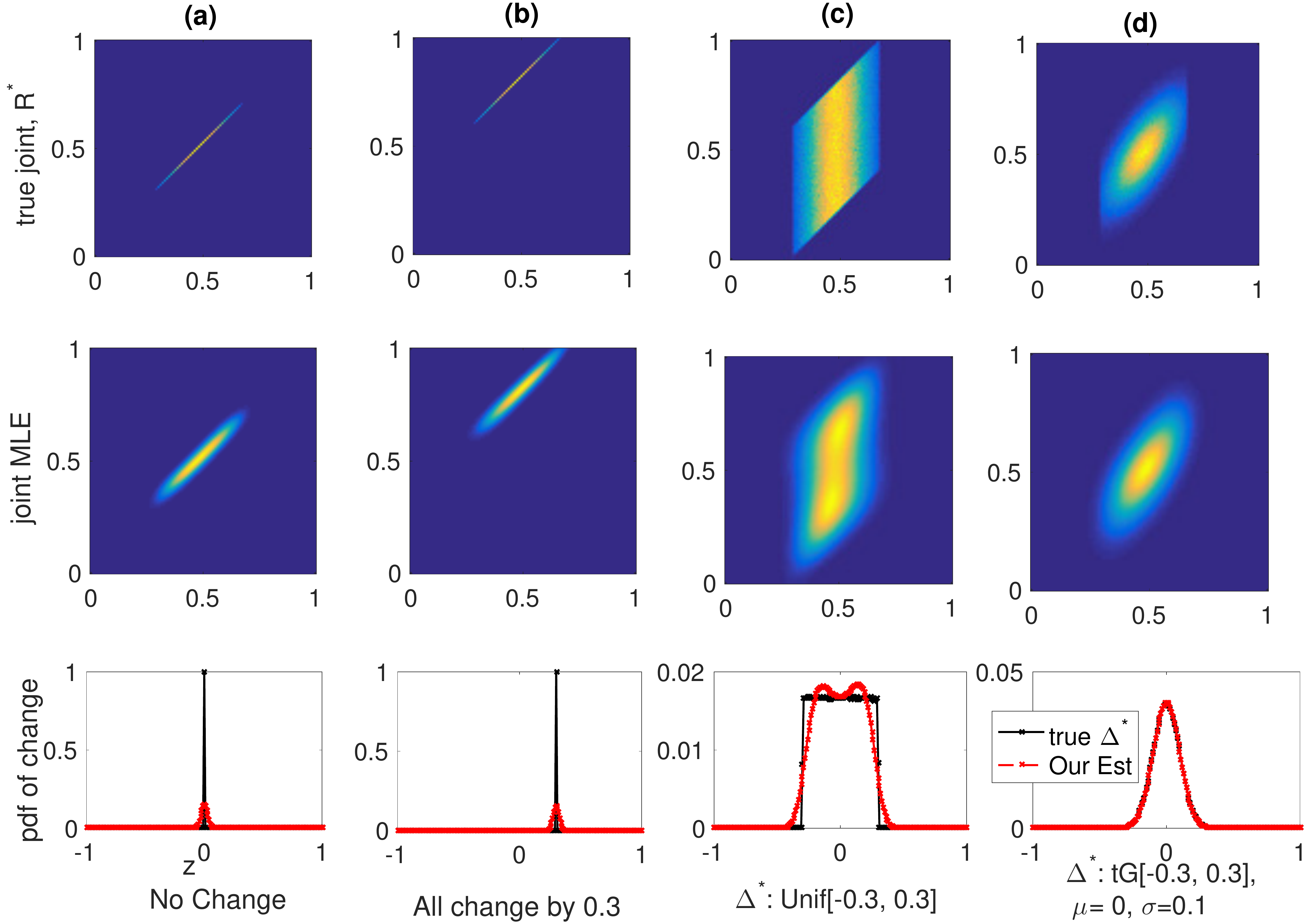}
    \caption{Empirical performance of the MLE in estimating the joint distribution (second row) and our estimator for distribution of change derived from the joint MLE (third row) for various true distributions (first row): (a) no change from before to after, (b) all the biases are increased by $0.3$, (c) trapezoid joint distribution where the change distribution is uniform on $[-0.3, 0.3]$ and (d) truncated Gaussian (tG) joint, with change distribution also a tG supported on $[-0.3, 0.3]$ with mean $\mu=0$ and standard deviation $\sigma=0.1$. N = 1e6, t = 5.}\label{fig:SimulationsJointDiffChange}
\end{figure*}

\paragraph{Estimation of the distribution of change:} Figure~\ref{fig:SimulationsJointDiffChange} shows the empirical performance of the MLE for estimating joint distribution of $(p, q)$ given the paired before and after observations (second row) and the derived estimate of the distribution of change (third row) for a variety of true joint distributions (first row). The number of coins is $N = 10^6$ and the number of tosses before and after is $t = 5$ for all the examples in the figure.

\paragraph{Estimation of the magnitude of change} Our estimation gets better as the number of observations per coin increases as the theory suggests. Figure~\ref{fig:VarytMagnitudeofChange} shows the error in estimating the distribution and the $l_1$-magnitude of change for varying number of tosses $t \in \{2, 3, 4, 5, 6, 8, 10, 12, 15, 20, 25, 30, 50 ,100\}$ for a variety of change distributions. The number of coins is $N=10^6$. Each experiment is run 10 times. For all the figures presented, the distribution of biases before was truncated Gaussian on $[0.3, 0.7]$ with $\mu=0.5$ and $\sigma = 0.1$. For comparison we have plotted estimators using the joint distribution estimated using the MLE (in red), the moment matching estimator (in green; also implemented using exponential gradient descent) and the naive empirical estimator (in blue).
We observed the same trend for other distributions for biases before, e.g., $3$-spikes of equal mass at $\{0.3, 0.5, 0.7\}$ and for the uniform distribution over $[0.3, 0.7]$. 

\paragraph{Paired test using the $l_1$-magnitude of change:} $l_1$-magnitude of change in the biases is a natural test statistic for nonparametric paired hypothesis testing. The null hypothesis is that there was no change in the biases from before to after.
That is, under the null, the change distribution is a spike at $0$ and consequently the $l_1$-magnitude of change is $0$. 
\begin{figure*}[b!]
 \centering
    \includegraphics[width=\textwidth]{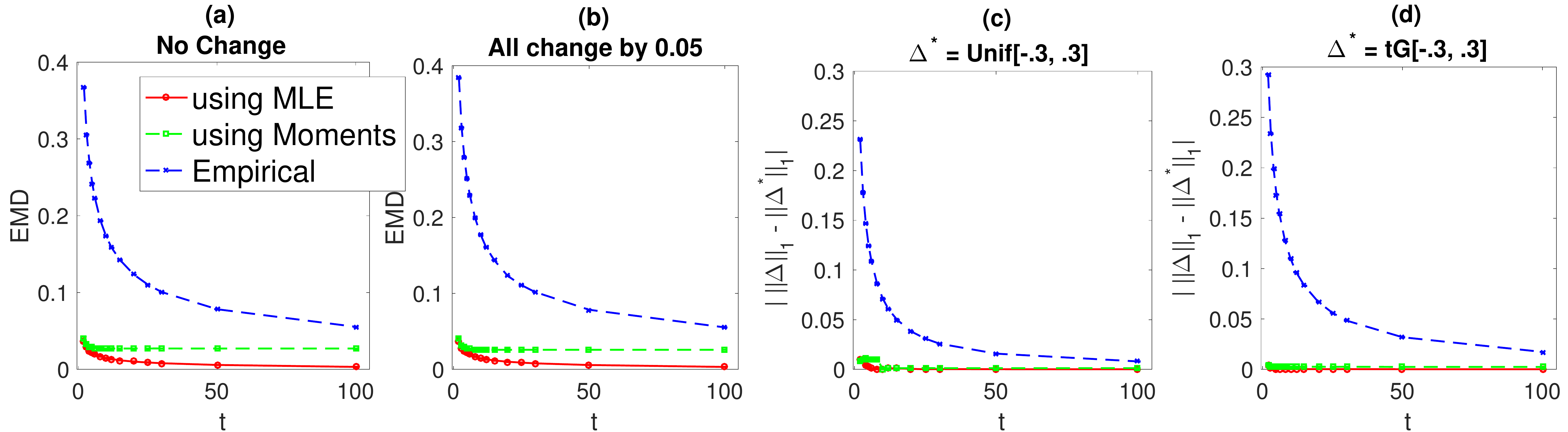}
    \caption{Error in estimating the distribution of change (measured in EMD) for various estimators as a function of varying number of tosses $t$ under (a) no change, (b) all biases increased by $0.05$, and $l_1-$magnitude of change for (c) $\Delta^\star$: Unif$[-0.3, 0.3]$, (d) $\Delta^\star$: tG$[-0.3, 0.3]$ with $\mu=0$, $\sigma=0.1$. N = 1e6. 10 experiments.}\label{fig:VarytMagnitudeofChange}
\end{figure*}
\begin{figure}[t!]
 \centering
 \includegraphics[width=0.75\columnwidth]{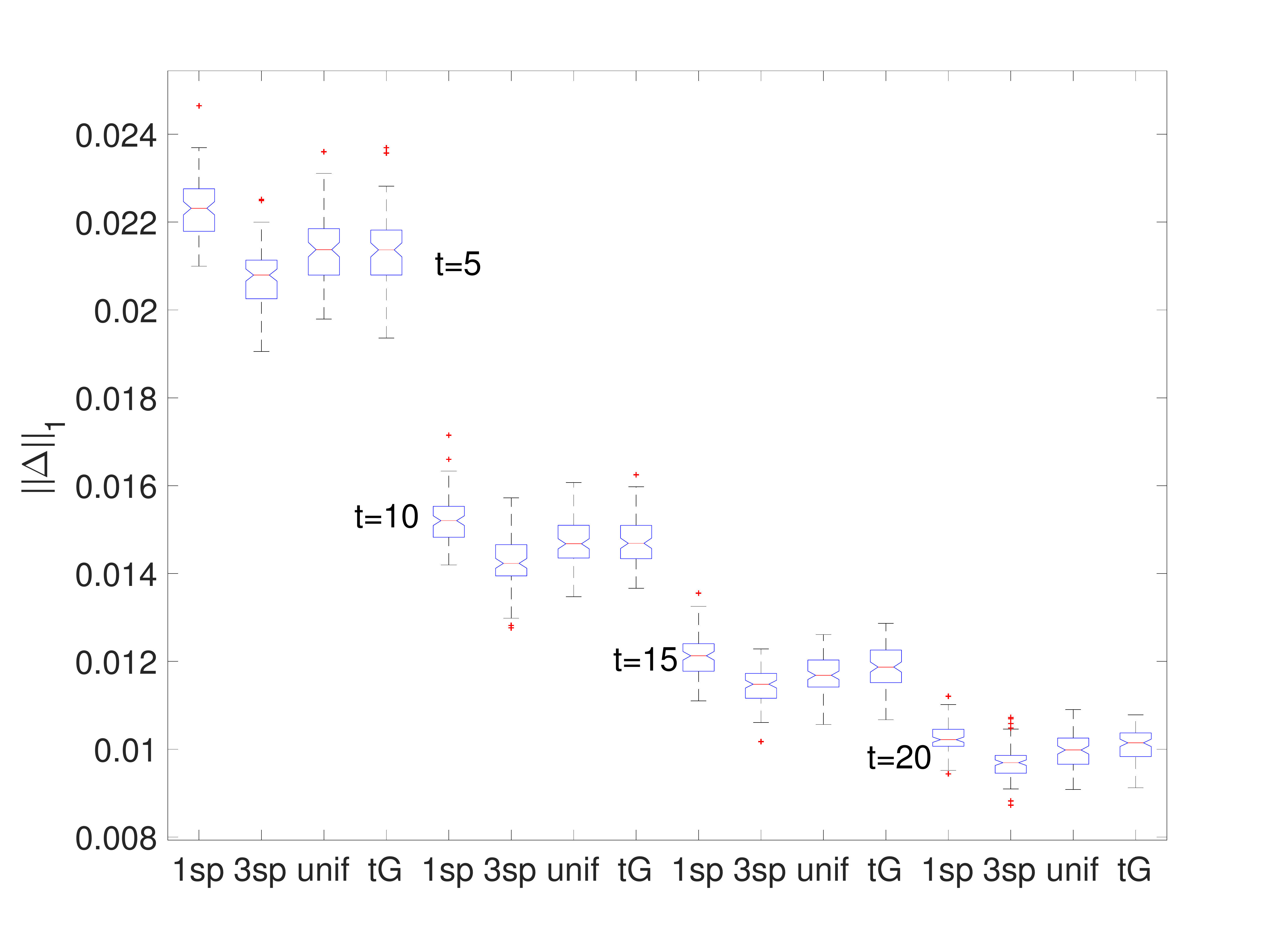}
 \caption{{\small{Estimated magnitude of change under the null hypothesis for $t=\{5, 10, 15, 20\}$ and for various before distributions: 1 spike at $0.5$ (1sp), 3 spikes (3sp), uniform (unif) and truncated Gaussian (tG). 100 Experiments, N=1e6. Box edges show $25$-th and $75$-th percentile, whisker edges are $\pm 2.7\text{std}$~\cite{MatlabBoxplot}.}}}
 \label{fig:HtestNullMagChange}
\end{figure}
Commonly used nonparametric paired difference tests like the sign test and the Wilcoxon's signed rank test are designed to test for change in mean or median. If the change in the biases in our setting is symmetric, e.g., uniform or truncated Gaussian around $0$, then both the sign test and the Wilcoxon's signed rank test fail to reject the null hypothesis. In contrast, using $\widehat{||\Delta||}_1$ as a test statistic we can successfully reject null hypothesis even when the change is symmetric around $0$.
Furthermore, while the sign test and the Wilcoxson's signed rank test can detect systematic shifts in change for large $N$, they do not provide a sense of how much change occurred. Often quantities like correlation, rank correlation and other metrics are estimated to infer the \emph{effect size}. In contrast, using the $l_1$-magnitude of change for testing provides a transparent and natural quantitative measure of the effect size.
Figure~\ref{fig:HtestNullMagChange} shows the $\widehat{||\Delta||}_1$ for various joint distributions that lead to the null hypothesis (no change) for $t = \{5, 10, 15, 20\}$ and $N = 10^6$. For a given $t$ and $N$, as long as the the $l_1$-magnitude of change is higher than the critical value (that can be picked via simulations by setting desired level of false positives admissible), hypothesis test using $\widehat{||\Delta||}_1$ can reject the null hypothesis.

\section*{Acknowledgements}
We thank our colleagues Jennifer Brennan, Maru Cabrera, Lalit Jain, John Thickstun and Jesse Thomason for the helpful discussions and comments. We also thank anonymous reviewers for their feedback and suggestions.

\bibliographystyle{alpha}
\bibliography{refs}

\appendix
\section{Proofs of Minimax Lower Bounds}
We will first prove the following lemma that will be used in the proof of information theoretic lower bound on the error in estimating the $\ell_1$-magnitude of the change.
\begin{proposition}\label{prop:prior-exist}
For any positive integer $s$, there exists a pair of distributions $P$ and $Q$ supported on $[a, b]$ where $0 < a < b$ such that $P$ and $Q$ have identical first $s$ moments, and $|\int_a^b |x-\frac{a+b}{2}|P(x)dx -\int_a^b |x-\frac{a+b}{2}|Q(x)dx|= \Omega(\frac{b-a}{s})$.
\end{proposition}
\begin{proof}
Our proof leverages Lemma 29 from~\cite{jiao2018minimax}, which is restated as follows,
\begin{proposition}[Lemma 29 of~\cite{jiao2018minimax}]\label{prop:jiaol1}
For any given even integer $s>0$, there exist two probability measures $P^*$ and $Q^*$ on $[-1, 1]$ that satisfy the following conditions:
\begin{enumerate}
\item $\int_{-1}^1P^*(x)x^ldx = \int_{-1}^1Q^*(x)x^ldx$, for $l = 0, 1, 2, \ldots , s$
\item $\int_{-1}^1 P^*(x)|x|dx - \int_{-1}^1 Q^*(x)|x|dx = 2E_s[|x|; [-1, 1]]$
\end{enumerate}
where $E_s[|t|; [-1, 1]]$ is the distance in the uniform norm on $[-1, 1]$ from the absolute value function $|t|$ to the space $\text{poly}_s$. It is known that $E_s[|x|;[-1,1]]=\Omega(1/s)$. 
\end{proposition}
Given the pair of distributions $P^*,Q^*$ constructed by Proposition~\ref{prop:jiaol1}, we define distribution $P(x) =  \frac{2}{b-a}P^*\left(\frac{x-(a+b)/2}{(b-a)/2}\right)$ and $Q(x) =  \frac{2}{b-a}Q^*\left(\frac{x-(a+b)/2}{(b-a)/2}\right)$. Since the transformation is linear, the first $s$ moments of $P,Q$ are identical. Further, $\int_a^b P(x)|x-(a+b)/2|dx-\int_a^b Q(x)|x-(a+b)/2|dx = \Omega\left(\frac{b-a}{s}\right)$.
\end{proof}

Minimax lower bound for estimating $\ell_1-$magnitude of change:
\begin{theorem}\label{thm:lowerboundtlogN}
Let $R$ be a distribution over $[0, 1]\times[0, 1]$. Let $\mathbf{S} := \{(X_i,Y_i)\}_{i=1}^N$ be random variables with $X_i \sim \text{Binomial}(t, p_i)$, $Y_i \sim \text{Binomial}(t, q_i)$ where $(p_i,q_i)$ is drawn independently from $R$. Define $\delta(R) = \tE_{(p,q)\sim R}[|q-p|]$. Let $\zeta$ be an estimator that maps $\mathbf{S}$ to a real value $\zeta(\mathbf{S})$. For every $t, N$ s.t. $t\le \frac{N^{2(e^4-1)}}{36}$, the following lower bound holds:
\begin{equation}
    \underset{\zeta}{\text{inf}}\ \underset{R}{\text{sup}}\ \text{E}\left[|\delta(R)-\zeta(\mathbf{S}))|\right] = \Omega(\frac{1}{\sqrt{t\log N}}).
\end{equation}
\end{theorem}
\begin{proof}
Let $U=\delta_{1/2}$ be a distribution supported only at $1/2$. We first apply Proposition~\ref{prop:prior-exist} to construct a pair of distributions $P^*$ and $Q^*$ supported on $\left[\frac{1}{2}-\sqrt{\frac{\log N}{t}}, \frac{1}{2}+\sqrt{\frac{\log N}{t}}\right]$ such that their first $L:=e^4\log N$ moments match, and $$\int_{\frac{1}{2}-\sqrt{\frac{\log N}{t}}}^{\frac{1}{2}+\sqrt{\frac{\log N}{t}}} |x-\frac{1}{2}|P(x)dx - \int_{\frac{1}{2}-\sqrt{\frac{\log N}{t}}}^{\frac{1}{2}+\sqrt{\frac{\log N}{t}}} |x-\frac{1}{2}|Q(x)dx \ge \Omega(\frac{1}{\sqrt{t\log N}}).$$ Let $p\sim P^*$ and $q\sim Q^*$. Let $X\sim \text{Binomial}(t,p)$ and $X' \sim \text{Binomial}(t,q)$. 
When use the following proposition from~\cite{vinayak2019maximum},
\begin{proposition}(Proposition 4.4 of~\cite{vinayak2019maximum})\label{prop:tv-single}
Let $P$ and $Q$ be two distributions, supported on $\left[\frac{1}{2}-\sqrt{\frac{\log N}{t}},\   \frac{1}{2}+\sqrt{\frac{\log N}{t}}\right]$, whose first $L:=e^4\log N$ moments match. Let $p\sim P$ and $q\sim Q$. Let $X\sim \text{Binomial}(t,p)$ and $Y\sim \text{Binomial}(t,q)$. Then the total variation distance between $X$ and $Y$ satisfies,
\[
    \text{TV}(X,Y) \le \frac{2\sqrt{t}}{N^{e^4}}.
\]
\end{proposition}
Using Proposition 4.4 of~\cite{vinayak2019maximum}, the total variation distance between $X$ and $X'$ satisfies,
\[
    \text{TV}(X,X') \le \frac{2\sqrt{t}}{N^{e^4}}.
\]  
Define $2$-dimensional distribution $P = U\times P^*, Q = U\times Q^*$. Let $(p_1,p_2)\sim P$ and $(q_1,q_2)\sim Q$. Let $X\sim \text{Binomial}(t,p_1), Y\sim \text{Binomial}(t,p_2)$ and $X' \sim \text{Binomial}(t,q_1), Y' \sim \text{Binomial}(t,q_2)$. Notice that $\delta(P) - \delta(Q) \ge \Omega(\frac{1}{\sqrt{t\log N}})$. Since the distribution of $X$ and $X'$ are identical, we have
\[
    \text{TV}\left((X,Y),(X',Y')\right) = \text{TV}\left(Y,Y'\right) \le  \frac{2\sqrt{t}}{N^{e^4}}.
\]
 Let $\mathbf{S} := \{(X_i,Y_i)\}_{i=1}^N$ be random variables with $X_i \sim \text{Binomial}(t, p_i)$, $Y_i \sim \text{Binomial}(t, q_i)$ where $(p_i,q_i)$ is drawn independently from $P$, and denote the joint distribution of $\mathbf{S}$ as $P_N$. Accordingly, let $\mathbf{S'} := \{(X_i,Y_i)\}_{i=1}^N$ be random variables with $X_i \sim \text{Binomial}(t, p_i)$, $Y_i \sim \text{Binomial}(t, q_i)$ where $(p_i,q_i)$ is drawn independently from $Q$, and denote the joint distribution of $\mathbf{S'}$ as $Q_N$. By the property of the product distribution, for $t\le \frac{N^{2(e^4-1)}}{36}$, $\text{TV}(P_N, Q_N) \le \frac{2\sqrt{t}}{N^{e^4-1}}\le \frac{1}{3}$, which implies that the minimax error is at least $\Omega(\frac{1}{\sqrt{t\log N}})$. 
 
 The lower bound for the case when $t = \mathcal{O}(\log{N})$ can be constructed very similarly. We leverage Proposition~\ref{prop:prior-exist-magnitude} and construct distributions similar to above supported on $[0, 1]$ whose first $t$ moments match while the $l_1-$magnitudes of change differ by $\frac{1}{t}$.
\end{proof}
Minimax lower bounds for estimating the distribution of change and the joint distribution are implied by the lower bounds on the estimation of $\ell_1-$magnitude of change. 
\section{Proofs of Upper Bounds}
Using the definition of $\Delta^\star$ (Equation~\eqref{eqn:changedistributiontrue}) and changing variables from $(x, z)$ to $(x, y)$ (Figure~\ref{fig:IlustrationdistOfChange}), for any function $f$, the expectation of $f$ under $\Delta^\star$ can be written as,
\begin{align}
\tE_{\Delta^\star}[f] &= \int_{z=-1}^1 f(z)\ d\Delta^\star(z)  =\int_{z=-1}^1 \int_{x=\max{\{0, -z\}}}^{\min{\{1, 1-z\}}} f(z)\ dR^\star(x, x+z)\\
&= \int_0^1 \int_0^1 f(y-x)\ dR^\star(x, y).
\end{align}
Similarly, from the definition of the estimator of distribution of change~\eqref{eqn:estchangedist1}, $
\tE_{\hDelta}[f] = \int_{x=0}^1 \int_{y=0}^1 f(y-x) d\hR_\mle(x, y)$. 
Therefore the Wasserstein-1 distance between the true distribution change $\Delta^\star$ and the estimated change distribution $\hDelta$ can be written as follows,
\begin{align}
W_1( \Delta^\star,  \hDelta)&= \underset{f \in \text{Lip}(1)}{\text{sup}} \left(\tE_{\Delta^\star}[f] - \tE_{\hDelta}[f] \right)\\
&=\underset{f \in \text{Lip}(1)}{\text{sup}} \int_{x=0}^1 \int_{y=0}^1 f(y-x) \left(dR^\star(x,y) - d\hR_\mle(x,y)\right).\label{eqn:W1ofDistofChangeA}
\end{align}
The estimator error for magnitude of change can be upper bounded as follows:
\begin{equation}
 \left|\widehat{||\Delta^\star||_1} - ||\Delta||_1\right| =\left|\int \int|y-x| (dR^\star(x,y) - d\hR_\mle(x,y)) \right|.\label{eqn:ErrorOfMagofChangeA}
\end{equation}
By the definition the Wasserstein-1 distance between the true joint distribution $R^\star$ and $\hR_\mle$ is,
\begin{equation}
    W_1( R^\star,  \widehat{R}_\mle)=\underset{f \in \text{Lip}(1)}{\text{sup}} \int_{0}^1 \int_{0}^1 f(x, y) (dR^\star(x, y) - d\hR_\mle(x, y)).\label{eqn:W1ofJointDistA}
\end{equation}
Noting that $|y-x|$ is a Lipschitz-1 function, all the estimator error expressions in Equations~\eqref{eqn:W1ofDistofChange},~\eqref{eqn:ErrorOfMagofChange} and~\eqref{eqn:W1ofJointDist} involve Lipschitz-1 functions.
The following bound holds for any $f \in \text{Lip}(1)$ supported on $[0, 1]^2$ and any polynomial approximation $\hf$ of $f$ on $[0, 1]^2$. 
\begin{align}
&\left| \int_{x=0}^1 \int_{y=0}^1 f(x, y) (dR^\star(x, y) - d\hR_\mle(x, y)) \right|,\nonumber\\
&= \left| \int_{0}^1 \int_{0}^1 \left(f(x, y) - \hf(x,y) \right) (dR^\star - d\hR_\mle) + \int_{0}^1 \int_{0}^1 \hf(x,y) (dR^\star - d\hR_\mle) \right| ,\nonumber\\
&\leq  2||f - \hf||_\infty + \left| \int_{0}^1 \int_{0}^1 \hf(x,y) (dR^\star - d\hR_\mle) \right| = 2||f - \hf||_\infty + \left|\tE_{ R^\star}[\hf] - \tE_{ \widehat{R}_\mle}[\hf]  \right|,\label{eqn:bound2dintegralA}
\end{align}
where $||f - \hat{f}||_\infty := \underset{(x,y) \in [0, 1]^2}{\max}|f(x, y) - \hf(x, y)|$ is the uniform polynomial approximation error. Let $\hf$ be the 2-d Bernstein polynomial of degree $2t$ defined by,
\begin{equation}
     \hf_t(x, y) = \sum_{u=0}^t \sum_{v=0}^t \kappa_{uv} B_{u}^t(x)  B_{v}^t(y), \label{eqn:BernsteinPolynomialApprox2dA}
\end{equation}
where $B_{u}^t(x) := \binom{t}{u}x^u (1-x)^{t-u}$ is the $u-$th Bernstein polynomial of degree $t$. 
Using this in Equation~\eqref{eqn:bound2dintegralA}, we can bound the EMD between $\Delta^\star$ and $\hDelta$ can be bounded as follows,
\begin{align}
W_1(\Delta^\star, \hDelta) \leq \underset{f \in \text{Lip}(1)}{\text{sup}} \left\{ 2 ||f - \hf||_\infty
+  \left| \sum_{u=0}^t \sum_{v=0}^t \kappa_{uv} \left(\tE_{ R^\star}[h_{uv}] - \tE_{ \widehat{R}_\mle}[h_{uv}]\right)  \right|  \right\},\label{eqn:W1DistChangeBoundA}
\end{align}
where $h_{uv}$ is the fraction of coins with observation tuple $(u,v)$. Similarly, the error bound on estimator of $l_1-$magnitude of change can be bounded as,
\begin{align}
\left|\widehat{||\Delta^\star||_1} - ||\Delta||_1\right| \leq  2 |\ |y-x| - \hf\ |_\infty
+  \left| \sum_{u=0}^t \sum_{v=0}^t \kappa_{uv} \left(\tE_{ R^\star}[h_{uv}] - \tE_{ \widehat{R}_\mle}[h_{uv}]\right)  \right|.\label{eqn:L1magBoundA}
\end{align}
The EMD between the true join distribution on $(X, Y)$, $R^\star$ and the MLE estimator $\hR_\mle$ can be bounded as,
\begin{align}
W_1(R^\star, \hR_{\mle}) \leq \underset{f \in \text{Lip}(1)}{\text{sup}} \left\{ 2 ||f - \hf||_\infty
+  \left| \sum_{u=0}^t \sum_{v=0}^t \kappa_{uv} \left(\tE_{ R^\star}[h_{uv}] - \tE_{ \widehat{R}_\mle}[h_{uv}]\right)  \right|  \right\},\label{eqn:W1JointBoundA}
\end{align}
The second term in the bounds above (Equation~\eqref{eqn:W1DistChangeBoundA},~\eqref{eqn:L1magBoundA} and~\eqref{eqn:W1JointBoundA}) can be bounded as follows,
\begin{align}
 &\left| \sum_{u,v}  \kappa_{uv} \left(\tE_{ R^\star}[h_{uv}] - \tE_{ \widehat{R}_\mle}[h_{uv}]\right)  \right| = \left| \sum_{u,v} \kappa_{uv} \left(\tE_{ R^\star}[h_{uv}] - h_{uv}^\obs + h_{uv}^\obs - \tE_{ \widehat{R}_\mle}[h_{uv}]\right)  \right|,\nonumber\\
 &\leq  \underset{(a)}{\underbrace{\left| \sum_{u,v} \kappa_{uv} \left(\tE_{ R^\star}[h_{uv}] - h_{uv}^\obs\right) \right|}} + \underset{(b)}{\underbrace{\left|\sum_{u,v} \kappa_{uv} \left(h_{uv}^\obs - \tE_{ \widehat{R}_\mle}[h_{uv}]\right)  \right|}},
\end{align}
where the term $(a)$ is the sampling error and term $(b)$ captures the error between the observed fingerprint the MLE solution. These two terms can be bounded by the following two lemmas which are extensions of their 1-d counterparts from~\cite{vinayak2019maximum}, 
\begin{lemma}(Concentration of 2-d fingerprint)\label{lemma:concfingerprintsA}
With probability at least $1-\alpha$, 
\begin{IEEEeqnarray}{rCl}
 \left| \sum_{u=0}^t \sum_{v=0}^t \kappa_{uv}\left( E_{R^\star}\left[ h_{u,v}\right] - h_{u,v}^\obs \right)\right| \leq \mathcal{O}\left( \max_{u,v} |\kappa_{uv}| \sqrt{\frac{\log{1/\alpha}}{N}}\right).
\end{IEEEeqnarray}
\end{lemma}
\begin{proof}
Recall that $h_{u, v}^{\text{obs}}$ is the fraction of the population that sees $u$ heads out of $t$ tosses before and $v$ heads out of $t$ tosses after. $\text{E}_{R^*}[h_{u, v}]$ is the expected fingerprint under the true distribution which is exactly $\text{E}[h_{u,v}^{\text{obs}}]$, and we will use $\text{E}[h_{u,v}^{\text{obs}}]$ and $\text{E}_{R^*}[h_{u,v}]$ interchangeably. Define, $\phi(X,Y) := \sum_{u=0}^t\sum_{v=0}^t \kappa_{u,v} \left( h_{u,v}^{\text{obs}} - \text{E}[h_{u, v}^{\text{obs}}] \right)$, that is,
\begin{IEEEeqnarray*}{rCl}
    \phi(X,Y) = \frac{1}{N}\sum_{i=1}^N \sum_{u=0}^t\sum_{v=0}^t \kappa_{u,v} \left( \mathbbm{1}_{\{(X_i, Y_i) = (u, v)\}} - \text{E}[h_{u, v}^{\text{obs}}] \right).
\end{IEEEeqnarray*}
Note that $\text{E}[\phi(X, Y)] = 0$. 
Let $\phi_{i'}(X, Y)$ be $\phi$ with $(X_i, Y_i)$ being replaced by re-drawing $(X_{i}', Y_{i}')$, with $X_i'\sim \text{Bin}(t, p_i)$ and $Y_i'\sim \text{Bin}(t, q_i)$. We can bound the difference in $\phi(X, Y)$ and $\phi_{i'}(X, Y)$ as follows,
\[ | \phi(X, Y) - \phi_{i'}(X, Y)  | \leq \underset{(u,v) }{\max}\ |\kappa_{uv}| \frac{2}{N}.  \]
By McDiarmid's Inequality, for some absolute constants $C,  c > 0$,
$$
    \text{Pr}\left( | \phi(X, Y) |  \geq \epsilon \right) 
 \leq  C\exp{\left( - \frac{c N \epsilon^2}{4 \left(\max_{(u,v)} |\kappa_{uv}|\right)^2} \right)}. 
$$
Hence, with probability at least $1-\delta$, 
\begin{IEEEeqnarray}{rCl}
 \left| \sum_{u=0}^t \sum_{v=0}^t \kappa_{uv}\left( E_{R^\star}\left[ h_{u,v}\right] - h_{u,v}^\obs \right)\right| \leq \mathcal{O}\left( \max_{u,v} |\kappa_{uv}| \sqrt{\frac{\log{1/\delta}}{N}}\right).
\end{IEEEeqnarray}
\end{proof}
\begin{lemma}(Bound on error in 2-d fingerprint due to MLE solution)\label{lemma:mlefingerprint}
For $3 \leq (t+1)^2 \leq \sqrt{C_0 N} + 2$, where $C_0 > 0$ is a constant, w. p. $1 - \alpha$,
{\small
\begin{IEEEeqnarray}{rCl}
\left| \sum_{u=0}^t  \sum_{v=0}^t \kappa_{uv} \left(h_{u, v}^{\text{obs}} - \text{E}_{R_\mle}[h_{u, v}] \right) \right| 
    &\leq& \max_{(u,v)} |\kappa_{uv}|\ \sum_{u=0}^t\sum_{v=0}^t \left| \left(h_{u,v}^{\text{obs}} - \text{E}_{R_\mle}[h_{u,v}] \right) \right|,\nonumber\\
    &\leq& \max_{(u,v)} |\kappa_{uv}|\ \sqrt{ 2\ \text{\text{ln}}(2) } \sqrt{\frac{(t+1)^2}{2N} \log{\frac{4N}{(t+1)^2}} + \frac{1}{N} \log{\frac{3e}{\alpha}} }.
\end{IEEEeqnarray} 
}
\end{lemma}
\begin{proof}
Let $\hR_{\mle}$ be an optimal solution to the MLE and $R^\star$ be the true distribution. By optimality of the MLE solution, we have the following inequality,
\begin{IEEEeqnarray}{rCl}
\label{eqn:mleOpt}
\KL(\h^{\obs}, \tE_{\hR_\mle}[\h]) \leq \KL(\h^{\obs}, \tE_{R^\star}[\h]).
\end{IEEEeqnarray}
\begin{proposition}[Pinsker's Inequality~\cite{cover2012elements}]
For discrete distributions $P$ and $Q$:
\begin{equation}
    \KL(P, Q) \geq \frac{1}{2\text{ln}2} || P - Q ||_1^2.
\end{equation}
\end{proposition}
Note that the fingerprint vector can be seen as a discrete distribution with support size $(t+1)^2$.
Using Pinsker's inequality and the optimality of the MLE solution,
\begin{IEEEeqnarray}{rCl}
\left| \sum_{u=0}^t  \sum_{v=0}^t \kappa_{uv} \left(h_{u, v}^{\text{obs}} - \text{E}_{\hR_\mle}[h_{u, v}] \right) \right|  &\leq& \sum_{u=0}^t  \sum_{v=0}^t |\kappa_{uv}|\ \left| h_{u,v}^{\text{obs}} - \text{E}_{\hR_\mle}[h_{u,v}] \right|, \nonumber\\
    &\leq& \max_{(u,v)} |\kappa_{uv}|\ \sum_{j=0}^t \left| h_j^{\obs} - \text{E}_{\hR_\mle}[h_j] \right|,\nonumber\\
    &\leq& \max_{(u,v)} |\kappa_{uv}|\ \sqrt{2\ \text{ln}(2)\ \KL(\h^{\obs}, E_{\hR_\mle}[\h])}, \\
    &\leq& \max_{(u,v)} |\kappa_{uv}|\ \sqrt{2\ \text{ln}(2)\ \KL(\h^{\obs}, E_{R^\star}[\h])}.
\end{IEEEeqnarray} 
Using the results on bounds on KL divergence between empirical observations and the true distribution for discrete distributions~\cite{mardia2018concentration},
for $3 \leq (t+1)^2 \leq \sqrt{C_0 N} + 2$, w. p. $1 - \delta$,
\begin{IEEEeqnarray*}{rCl}
\left| \sum_{u=0}^t  \sum_{v=0}^t \kappa_{uv} \left(h_{u, v}^{\text{obs}} - \text{E}_{R_\mle}[h_{u, v}] \right) \right| 
    \leq  \max_{(u,v)} |\kappa_{uv}|\ \sqrt{ 2\ \text{ln}(2) } \sqrt{\frac{(t+1)^2}{2N} \log{\frac{4N}{(t+1)^2}} + \frac{1}{N} \log{\frac{3e}{\delta}} }.
\end{IEEEeqnarray*} 
\end{proof}
Note that the term $ \max_{(u,v)} |\kappa_{uv}|$ appears in the bounds on both the terms $(a)$ and $(b)$. Therefore, it is crucial to bound the uniform approximation error $||f - \hf||_\infty$ while keeping the coefficients of the approximation~\eqref{eqn:BernsteinPolynomialApprox2d} from being too large. Polynomial approximation of Liptschitz-1 functions using Bernstein polynomials play a crucial role in obtaining the minimax optimal rates in Theorem~\ref{thm:upperbounds} and Proposition~\ref{prop:MLEJointW1}.
\begin{lemma}(Bernstein polynomial approximation of 2-d Lipschitz-1 functions)
\label{lem:coeffbound}
Any Lipschitz-1 function on $[0, 1]^2$ can be approximated using Bernstein polynomials (Equation~\eqref{eqn:BernsteinPolynomialApprox2d}) of degree $2t$, with an uniform approximation error of
\begin{itemize}[leftmargin=*, topsep=0pt, itemsep=0pt]
    \item $\mathcal{O}(\frac{1}{t})$ with $\underset{(u,v)}{\max}\ |\kappa_{uv}| \leq t 2^{2t}$.
    \item $\mathcal{O}(\frac{1}{k})$ with $\underset{(u, v)}{\max}\ |\kappa_{uv}| \leq k (t+1)^2 e^{\frac{2k^2}{t}}$, for $k < t$.
\end{itemize}
\end{lemma}
The proof of Lemma~\ref{lem:coeffbound} uses the following results. 
\begin{lemma}(\cite{bagby2002multivariate} Polynomial approximation of Continuous Functions on $[0, 1]^2$)\label{lemma:exist-poly2d}
Given any continuous function $f:[0,1]^2 \rightarrow \mathbb{R}$, there exists a degree $k$ polynomial,
\begin{equation}
    p_k(x, y):= \underset{: i+j \leq k}{\sum_{i=0}^k\sum_{j=0}^k} c_{ij} x^i y^j,\  \forall (x, y) \in [0, 1]^2\label{eqn:polyapprox2d}
\end{equation} that approximates $f$ with error $\max_{(x, y)\in [0,1]^2}|f(x, y)-p_k(x,y)| = O(\frac{1}{k})$. 
\end{lemma}
Let $\tT_m$ denote Chebyshev polynomial of degree $m$ shifted to $[0, 1]$ which satisfy the following recursive relation: 
\[\tT_m(x) = (4x - 2)\tT_{m-1} - \tT_{m-2}(x),\ m= 2, 3,....,\]
and $\tT_0(x) = 1$, $\tT_1(x) = 2x - 1$.
\begin{lemma}(Chebyshev polynomial approximation of Lipshitz-1 functions on $[0, 1]^2$)\label{lemma:cheby-approx}
Given any Lipschitz-1 function $f:[0, 1]^2 \rightarrow \mathbb{R}$, there exists a degree $k$ polynomial in the form of 
\begin{equation}
    f_k(x, y)=\underset{:m+l \leq k}{\sum_{m=0}^k \sum_{l=0}^k} \tau_{m,l}\tT_m(x) \tT_l(y), \label{eqn:chebyshevapprox2d}
\end{equation}that approximates $f(x, y)$ with uniform approximation error of $||f-f_k||_\infty = O(\frac{1}{k})$, where $\tT_m(x)$ denotes Chebyshev polynomial of degree $m$ shifted to $[0, 1]$. Further, the coefficient vector satisfies $\|\tau\|_2\le 32$.
\end{lemma}
\begin{proof}
Chebyshev polynomials of degree up to $k$ form a basis for polynomials of degree $k$. Therefore each term $x^iy^j$ in the approximation polynomial~\eqref{eqn:polyapprox2d} can be written as a linear combination of Chebyshev polynomials of degree up to $k$,
\begin{equation}
    x^i y^j = \left(\sum_{m=0}^i a_{m, i} \tT_m(x)\right) \left(\sum_{n=0}^j b_{n, j} \tT_n(y)\right) = \sum_{m,n = 0}^{i, j} a_{m, i} b_{n, j} \tT_m(x) \tT_n(y).
\end{equation}
The approximation polynomial in Lemma~\ref{lemma:exist-poly2d} (Equation~\eqref{eqn:polyapprox2d}) can be transformed to Chebyshev polynomial basis,
\begin{equation}
p_k(x,y) = \underset{:m+l \leq k}{\sum_{m=0}^k \sum_{l=0}^k} \tau_{m,l}\tT_m(x) \tT_l(y) =: f_k(x, y),\label{eqn:cheybshevapprox}
\end{equation}
where, $\tau_{m,l} = \underset{:i+j\leq k}{\sum_{i=m}^k \sum_{j=l}^k} c_{i, j} a_{m,i} b_{l, j}.$
Therefore, $f_k(x, y)$ is a uniform $\frac{1}{k}-$ approximation of $f(x, y)$.
Chebyshev polynomials (shifted) form a sequence of orthogonal polynomials with respect to the weight $\frac{1}{\sqrt{4x - 4x^2}}$:
\begin{equation}
    \int_{0}^1 \tT_m(x) \tT_n(x) \frac{dx}{\sqrt{4x - 4x^2}} = \left\{
	\begin{array}{ll}
		0  & \mbox{if } m \neq n \\
		\frac{\pi}{2} & \mbox{if } m = n = 0\\
		\frac{\pi}{4} & \mbox{if } m = n \neq 0.
	\end{array}
\right.
\end{equation}
Therefore, we can bound the coefficients of Cheybyshev polynomials in Equation~\eqref{eqn:cheybshevapprox},
\begin{equation*}
\int\int |f_{k}(x, y)|^2\frac{dx}{\sqrt{4x - 4x^2}}\frac{dy}{\sqrt{4y - 4y^2}}= \frac{\pi^2}{16}\left(4\tau_{0,0}^2  + 2\sum_{l=0}^k \tau_{0,l}^2 + 2\sum_{m=0}^k \tau_{m,0}^2+ \underset{:m+l \leq k}{\sum_{m=1}^k \sum_{l=1}^k} \tau_{m,l}^2 \right).
\end{equation*}
Since $f$ is Lipschitz-1 on $[0,1]^2$, w.l.o.g $f(x, y) \leq \sqrt{2}$ for all $(x, y)$ in $[0, 1]^2$. Since $f_k(x, y)$ is a $\frac{1}{k}-$ uniform approximation of $f$ on $[0, 1]^2$, we have $|f_{k}(x, y)| \leq \sqrt{2} + \frac{1}{k} \leq 2\sqrt{2}$ for all $(x, y)$ in $[0, 1]^2$. Therefore,
\begin{IEEEeqnarray}{rCl}
\int\int |f_{k}(x, y)|^2\frac{dx}{\sqrt{4x - 4x^2}}\frac{dy}{\sqrt{4y - 4y^2}} \leq 2\pi^2.
\end{IEEEeqnarray}
Thus we can bound the $l_2-$norm of the coefficients $\tau$ in the approximation~\eqref{eqn:cheybshevapprox},\[
3\ \tau_{0,0}^2 + \sum_{l=0}^k \tau_{0,l}^2+ \sum_{m=0}^k \tau_{m,0}^2 + ||\tau||^2_2  \leq 32.\]
Therefore, $||\tau||_2^2 \leq 32$. 
\end{proof}
Note that $\tau$ is a $\binom{k+2}{2}$ length vector. Hence, we have the following bound,
\begin{equation}
    ||\tau||_1 = \underset{:m+l \leq k}{\sum_{m=0}^k \sum_{l=0}^k} |\tau_{m,l}| \leq \sqrt{32} \sqrt{\binom{k+2}{2}} < 4(k+2).
\end{equation}

\textbf{Transforming to Bernstein polynomials:} Bernstein polynomials of degree $m < t$ can be raised to degree $t$ as:
\begin{equation}
    B_{i}^m(x) = \sum_{j=i}^{i + t - m} \frac{\binom{m}{i} \binom{t-m}{j-i}}{\binom{t}{j}} B_{j}^t(x).
\end{equation}
Using degree raising of Bernstein polynomials, we can write shifted Chebyshev polynomials of degree $m < t$ in terms of Bernstein polynomials of degree $t$ as~\cite{Rababah2003},
\begin{equation}
 \tT_{m}(x) = \sum_{i=0}^m \left(-1\right)^{m-i} \frac{\binom{2m}{2i}}{\binom{m}{i}} \sum_{j=i}^{i + t - m} \frac{\binom{m}{i} \binom{t-m}{j-i}}{\binom{t}{j}} B_{j}^t(x) =: \sum_{j=0}^t C(t, m, j) B_{j}^t(x),
\end{equation}
where the coefficient of $j$-th Bernstien polynomial of degree $t$ is given by,
\begin{IEEEeqnarray}{rCl}
C(t, m, j) : = \sum_{l=0}^{j} \left(-1\right)^{m-l} \frac{\binom{2m}{2l} \binom{t-m}{j-l}}{\binom{t}{j}},
\end{IEEEeqnarray}
with $\binom{a}{b} = 0$ when $b > a > 0$. The following lemma from from~\cite{vinayak2019maximum} provides an upper bound on the coefficients $C(t, m, j)$,
\begin{lemma}\label{lem:l2normbound}
The $l_2$-norm of the coefficients of $B_j^t$ can be bounded as follows,
\begin{equation}
\label{eqn:l2normbound}
    \sqrt{\sum_{j=0}^t |C(t, m, j)|^2} \leq (t+1) e^{\frac{m^2}{t}}.
\end{equation}
And, hence the coefficients of $B_j^t$ can be bounded as follows,
\begin{IEEEeqnarray}{rCl}
\label{eqn:coeffbound}
|C(t, m, j)| \leq (t+1) e^{\frac{m^2}{t}}.
\end{IEEEeqnarray}
\end{lemma}
Every term in the Chebyshev approximation polynomial of degree $k$ in Equation~\eqref{eqn:chebyshevapprox2d} can therefore be written using Bernstein polynomials of degree up to $t$,\begin{IEEEeqnarray}{rCl}
 f_k(x, y) &=& \underset{:m+l \leq k}{\sum_{m=0}^k \sum_{l=0}^k }\tau_{m,l}\tT_m(x) \tT_l(y) = \underset{:m+l \leq k}{\sum_{m=0}^k \sum_{l=0}^k } \tau_{m,l} \left( \sum_{j=0}^t C(t, m, j) B_{j}^t(x) \right) \left( \sum_{v=0}^t C(t, l, v) B_{v}^t(y) \right), \nonumber\\
&=:& \sum_{j=0}^t \sum_{v=0}^t \kappa_{jv} B_{j}^t(x)  B_{v}^t(y),
\end{IEEEeqnarray}
where the coefficients are given by,
\begin{equation}
    \kappa_{jv} := \underset{:m+l \leq k}{\sum_{m=0}^k \sum_{l=0}^k} \tau_{m,l} C(t, m, j)C(t, l, v).
\end{equation}
Each of these coefficients can be bounded by,
\begin{IEEEeqnarray}{rCl}
\label{eqn:boundCoeffOverall}
 |\kappa_{jv}|
 &\leq&  \underset{:m+l \leq k}{\sum_{m=0}^k \sum_{l=0}^k} |\tau_{m,l}|\ \left| C(t, m, j) \right| \left| C(t, l, v) \right|\leq  \max_{m,l} |C(t, m, j)| | C(t, l, v)| \left(\underset{:m+l \leq k}{\sum_{m=0}^k \sum_{l=0}^k} |\tau_{m,l}| \right) \nonumber\\
 &\leq&  (t+1)^2e^{\frac{2k^2}{t}} \left(\underset{:m+l \leq k}{\sum_{m=0}^k \sum_{l=0}^k} |\tau_{m,l}| \right) \leq 4(k+2)(t+1)^2e^{\frac{2k^2}{t}}.
 \label{eqn:boundonbj}
\end{IEEEeqnarray}
This concludes the proof of Lemma~\ref{lem:coeffbound}.

Putting these results together, 
\begin{IEEEeqnarray}{rCl}
W_1\left(R^\star, \hR_\mle  \right)
&\leq& \mathcal{O}\left( \frac{1}{k} \right) +\mathcal{O}\left( \max_{u,v} |\kappa_{uv}| \sqrt{\frac{\log{1/\delta}}{N}}\right),\nonumber\\
&&+ \mathcal{O}\left( \max_{(u,v)} |\kappa_{uv}|\  \sqrt{\frac{(t+1)^2}{2N} \log{\frac{4N}{(t+1)^2}} + \frac{1}{N} \log{\frac{3e}{\delta}}} \right),\\
&\leq& \mathcal{O}\left( \frac{1}{k} \right) +\mathcal{O}\left( \max_{u,v} |\kappa_{uv}| \sqrt{\frac{\log{1/\delta}}{N}}\right),\nonumber\\
&&+ \mathcal{O}\left( \max_{(u,v)} |\kappa_{uv}|\  \sqrt{\frac{(t+1)^2}{2N} \log{\frac{4N}{(t+1)^2}} + \frac{1}{N} \log{\frac{3e}{\delta}}} \right),
\end{IEEEeqnarray}
where $k = t$ for the case when $t = \mathcal{O}(\log{N})$, and $ k = \sqrt{t\ \log{N^c}}$ when $t > \log{N}$. These bounds are also upper bounds for $W_1\left(\Delta^\star, \hDelta  \right)$ and $\| || \Delta^\star ||_1 - ||\hDelta\||_1 \|$. This completes the proof of Proposition~\ref{prop:MLEJointW1} and the main results stated in Theorem~\ref{thm:upperbounds}.

\end{document}